%% file: main.tex
\documentclass[runningheads]{llncs}

\usepackage[utf8]{inputenc} 
\usepackage[T1]{fontenc}
\usepackage{graphicx}
\usepackage{bm}
\usepackage{bbm}

\usepackage[libertine,cmintegrals,cmbraces,vvarbb]{newtxmath}
\usepackage{enumitem}
\usepackage{float}
\usepackage{dsfont}
\usepackage{xcolor}
\usepackage{algorithm, algpseudocode}
\usepackage{wrapfig}
\usepackage{multicol, multirow}    
\usepackage{hyperref}
\usepackage{longtable}
\usepackage{supertabular}
\usepackage{textcomp}
\usepackage{caption}
\usepackage{subcaption}
\newcommand{\cupdot}{\mathbin{\dot{\cup}}} 

\usepackage{soul}

\usepackage{environ}

\RenewEnviron{proof}{%
\par\noindent\textit{Proof. }\BODY\hfill$\Box$\par
}

\newcommand{\set}[1]{\mathcal{#1}}
\newcommand{\da}{\mathbin{:=}}
\newcommand{\abb}{\mathbin{:}}
\newcommand{\lmm}{\{\!|}
\newcommand{\rmm}{|\!\}}
\newcommand{\RM}[1]{\MakeUppercase{\romannumeral #1{.}}}

\begin{document}
\title{Weisfeiler–Lehman meets Events:\\ An Expressivity Analysis for Continuous-Time Dynamic Graph Neural Networks}
\titlerunning{Weisfeiler-Lehman meets Events}
\author{\href{https://orcid.org/0000-0002-2984-2119}{Silvia Beddar-Wiesing}\thanks{Corresponding author. Both authors contributed equally.} \and
\href{https://orcid.org/0000-0001-7912-0969}{Alice Moallemy-Oureh}}
\authorrunning{S.~Beddar-Wiesing and A.~Moallemy-Oureh}
%
\institute{University of Kassel, Germany \\
\email{\{s.beddarwiesing, amoallemy\}@uni-kassel.de}
}
\maketitle              

\input{chapter/abstract}

\input{chapter/intro}
\input{chapter/foundations}

\input{chapter/preliminaries}

\input{chapter/apx_continuous_time}

\input{chapter/conclusion}


\bibliographystyle{splncs04}
\bibliography{bibliography} 

\end{document}

%% file: chapter/abstract.tex
\begin{abstract}

Graph Neural Networks (GNNs) are known to match the distinguishing power of the 1-Weisfeiler-Lehman (1-WL) test, and the resulting partitions coincide with the unfolding tree equivalence classes of graphs. Preserving this equivalence, GNNs can universally approximate any target function on graphs in probability up to any precision. However, these results are limited to attributed discrete-dynamic graphs represented as sequences of connected graph snapshots. Real-world systems, such as communication networks, financial transaction networks, and molecular interactions, evolve asynchronously and may split into disconnected components. In this paper, we extend the theory of attributed discrete-dynamic graphs to attributed continuous-time dynamic graphs with arbitrary connectivity. To this end, we introduce a continuous-time dynamic 1-WL test, prove its equivalence to continuous-time dynamic unfolding trees, and identify a class of continuous-time dynamic GNNs (CGNNs) based on discrete-dynamic GNN architectures that retain both distinguishing power and universal approximation guarantees. Our constructive proofs further yield practical design guidelines, emphasizing a compact and expressive CGNN architecture with piece-wise continuously differentiable temporal functions to process asynchronous, disconnected graphs.

\keywords{Dynamic Graph Neural Network  \and Continuous-Time Dynamic Graphs \and GNN Expressivity \and Disconnected Graphs \and Weisfeiler-Lehman Test}
\end{abstract}

%% file: chapter/intro.tex
\section{Introduction}
The expressive power of Graph Neural Networks (GNNs) can be viewed through their ability to distinguish non‐isomorphic graphs or their capacity to approximate target functions defined on graph‐structured data. It is already well investigated that standard Message‐Passing GNNs cannot distinguish more graphs than the 1–Weisfeiler–Lehman (1–WL) test \cite{xu2018powerful}. This result has recently been extended to attributed and discrete‐time dynamic settings in \cite{beddar2024weisfeiler}\footnote{We build on \cite{beddar2024weisfeiler} and follow its notation and preliminaries. To keep the presentation concise, we do not recapitulate them here and suggest consulting that work first.}.  

Further, using the Stone–Weierstraß theorem, GNNs have also been shown to serve as universal approximations \cite{azizian2021expressive}. More recent work additionally provides constructive universality results for both attributed and discrete‐dynamic graphs by exploiting the equivalence between 1–WL and unfolding‐tree partitions \cite{d2021unifying,beddar2024weisfeiler}.  

However, many practical systems evolve in continuous time and may fragment into disconnected components, as seen in asynchronous interaction streams, communication networks, or chemical reaction processes \cite{arx_rossi_2020,inproc_trivedi_2019,nikitin2020dracon}. Existing theory does not account for such continuous‐time dynamics or arbitrary connectivity patterns, limiting our understanding of GNNs in these settings.

In this paper, we extend both distinguishing and approximation guarantees to continuous‐time dynamic graphs (CTDGs), which are defined in Sec.~\ref {section_foundation}. First, in Sec.~\ref{section_preliminaries} we introduce a continuous-time analogue of the 1–WL test for CTDGs and define corresponding continuous‐time unfolding trees. Then, we prove the correspondence between the continuous-time 1-WL and unfolding tree equivalence. Second, we leverage this correspondence to constructively show in Sec.~\ref{section_universal_apx_CGNN} that a broad class of continuous‐time dynamic GNNs (CGNNs) exactly inherits the distinguishing power of continuous-time 1–WL and achieves universal approximation for any measurable function respecting the induced equivalence in probability and to any precision. Finally, we discuss limitations and conclude our results in Sec.~\ref{section_conclusion}.\\
\begin{center}
\flushleft\textbf{Notation}\\
\begin{minipage}[t]{.48\linewidth}
\small
    \begin{center}
\tablehead{\hline\textbf{Notation} &\\\hline}\label{tab_notation}
\tabletail{\hline}
\begin{tabular}{|l|l|}
	\hline
	$\mathbb{N}_0$ & natural numbers starting at $0$\\
    $\mathbb{R}_{\geq 0}$ & non negative real numbers\\
	$\mathbb{R}^k$ & $\mathbb{R}$ vector space of dimension $k$\\
	\hline
	$\mathbb{0}$ & zero vector\\
	$|a|$ & absolute value of a real $a$ \\
	$\|\cdot \|$ & norm on $\mathbb{R}$ \\
	$\|\cdot \|_\infty$ & $\infty$-norm on $\mathbb{R}$ \\
	$|M|$ & number of elements of a set $M$\\ 
	$[n], \,n\in \mathbb{N}$ & sequence $1,2,\ldots, n$ \\
	$[n]_{0}, \,n\in \mathbb{N}_{0}$ & sequence $0,1,\ldots, n$ \\
	\hline
\end{tabular}
\end{center}
\end{minipage}
\hspace{.01\linewidth} 
\begin{minipage}[t]{.48\linewidth}
\small
\begin{center}
\tablehead{\hline\textbf{Notation} &\\\hline}\label{tab_notation}
\tabletail{\hline}
\begin{tabular}{|l|l|}
\hline
	$\perp$ & undefined; non-existent element\\
	$\{\cdot\}$ & set \\
	$\{\!\vert\cdot\vert\!\}$ & multiset, i.e. set allowing \\
	& multiple appearances of entries\\
	$(x_i)_{i \in I}$ & vector of elements $x_i$ \\
	& for indices in set $I$\\
	$[v|w]$ & stacking of vectors $v,w$ \\
	\hline
	$\subseteq$ & sub(multi)set \\
	$\subset$ & proper sub(multi)set \\
	$M \times N$ & factor set of two sets $M$ and $N$\\
	\hline
\end{tabular}
\end{center}
\end{minipage}
\end{center}

%% file: chapter/foundations.tex
\section{Foundations}\label{section_foundation}

In this section, we formalize continuous‐time dynamic graphs, introduce the corresponding notion of graph isomorphism, and present a generic GNN architecture often
used for processing them.

\begin{definition}[continuous-time dynamic Graph]\label{def_continuous_dynamic_graph}
Let ${\set{T}=[t_0, \ldots, t_T]\subset \mathbb{R}_{\geq 0}}$ be a finite set of timestamps. Then a \textbf{continuous-time dynamic graph} (CDG) ${G=(g_0,\set{O})}$ includes a start graph $g_0=(\set{V}_0, \set{E}_0, \alpha_0, \omega_0)$ with finite node set $\set{V}_0$, edges $\set{E}_0$, corresponding attribute mappings $\alpha_0, \omega_0$, and a set ${\set{O}=\{o_0,\ldots,o_T\}}$ of events. For all $i = 0, \ldots, T \in \mathbb{N}_0$ an event $o_i=(x, k, \gamma, t_i)$ determines on which item $x$ (node or edge) with attribute $\gamma\in \set{A}$\footnote{According to \cite{beddar2024weisfeiler}, without loss of generality, we consider $\set{A} \subseteq \mathbb{R}^d, d \in \mathbb{N}$ to be the common attribute space for nodes and edges at all times.} an event of type $k$ (addition, deletion or attribute change) happens at time $t_i$. 
\end{definition} 

For the internal representation of continuous-time dynamic graphs in GNNs, it is often necessary to also store the global structural information. Therefore, let $g_i=(\set{V}_i, \set{E}_i, \alpha_i, \omega_i)$ 
be the \textbf{current graph} at time $t_i$ obtained by updating the start graph with the events $o_j$ up to time $t_i$. Further, we refer to the \textbf{domain of CDGs} with bounded node and timestamp sets as $\set{G}$.


\begin{definition}[Graph Isomorphism]\label{def_graph_iso}
Let ${G=(g_0,\set{O})}$ and ${G'=(g_0',\set{O}')}$ be two CDGs. Further, let $g_i=(\set{V}_i,\set{E}_i,\alpha_i,\omega_i)$ and $g'_i=(\set{V}'_i,\set{E}'_i,\alpha'_i,\omega'_i)$ be the current graphs of $G$ and $G'$ at time $t_i$, respectively.
Then, $G$ and $G'$ are \textbf{isomorphic}, i.e., $G \approx G'$, if and only if there exist 
bijective functions $\phi_i\abb \set{V}_i\rightarrow\set{V}'_i$ and $\psi_i\abb \set{A}\rightarrow \set{A}'$, such that for all $i = 0, \ldots, T \in \mathbb{N}_0$ it holds
 \begin{itemize}
    \item $v \in \mathcal{V}_i \, \Leftrightarrow \, \phi_i(v) \in \mathcal{V}_i'\quad\forall \; v \in \mathcal{V}_i$,
    \item ${\{u, v\} \in \mathcal{E}_i \, \Leftrightarrow \, \{\phi_i(u), \phi_i(v)\} \in \mathcal{E}'_i \quad \forall \; \{u,v\} \in \mathcal{E}}_i$,\vspace{.2cm}
    \item $\psi_i\left(\alpha_i(v)\right) = \alpha_i'(\phi_i(v)) \quad \forall \; v\in \mathcal{V}_i$,
    \item ${\psi_i\left(\omega_i(\{u, v\})\right) = \omega_i'(\{\phi_i(u), \phi_i(v)\}) \quad \forall \; \{u, v\} \in \mathcal{E}_i}$.
 \end{itemize}
\end{definition}

Processing graphs in continuous-time faces the challenge that only the local changes are explicitly given, while the global context is only implicitly present. However, several GNNs manage the global information by internally storing the current graphs per event. This allows for a GNN architecture similar to that for discrete dynamic graphs, while considering the actual timestamps.

\begin{definition}[Continuous Dynamic GNN (CGNN)]\label{CGNN}
Given a CDG ${G\in\set{G}}$, a \textbf{Continuous-Time Dynamic Graph Neural Network (CGNN)} is defined using a recursive function $f$ that is (piece-wise) continuously differentiable (between consecutive events) for temporal modeling as follows.

Let $N=|\set{V} := \bigcup_{t_i\in\set{T}}\set{V}_i|$ 
be the maximal number of nodes appearing in the CDG. Further, let \texttt{SGNN} be a GNN that is a universal approximator on static attributed undirected homogeneous graphs (SAUHGs) as defined in \cite{beddar2024weisfeiler}. Then, the start embeddings at $t_0$ are  determined for all nodes $v\in\set{V}$\footnote{If $v\notin\set{V}_i$ for a timestamp $t_i$, then the \texttt{SGNN} returns empty node and state embeddings $\bm{h}_v(t_i)=\bm{q}_v(t_i)=\perp$, see \cite{beddar2024weisfeiler}.} as
\begin{equation*}
        \bm{q}_1(t_0),\ldots,\bm{q}_N(t_0) \da \bm{h}_1(t_0),\ldots,\bm{h}_N(t_0) \da 
            \texttt{SGNN}(g_{0}).
\end{equation*}

The hidden state representations $\bm{q}_v(t_i)\in\mathbb{R}^{s}$ for all nodes $v\in\set{V}$ at timestamps $t_i,\ i> 0$ are calculated using the hidden node representations $\bm{h}_v(t_i)\in \mathbb{R}^{r}$, with $r=s$\footnote{In general, the hidden dimensions may be unequal; however, to ensure the correct dimension, an MLP must be incorporated into the initialization of $\bm{q}_v(t_0)$ for all $v\in\set{V}$.}, given by the SGNN as follows:
\begin{align*}
        \bm{q}_v(t_i) & \da  f(\bm{q}_v(t_{i-1}), \bm{h}_v(t_i)), \\
        \bm{h}_1(t_i),\ldots,\bm{h}_N(t_i) & \da \texttt{SGNN}(g_{i}),
\end{align*}
where ${f:\mathbb{R}^{s\times r} \rightarrow \mathbb{R}^{s}}$ is a neural architecture for temporal modeling (in the methods surveyed in \cite{skarding2021foundations}, $f$ is usually an RNN or an LSTM). 

The stacked version over all nodes of the CGNN is then:
\begin{equation*}\label{stack_eq_cont}
        \bm{Q}(t_0)  = \bm{H}(t_0) = \texttt{SGNN}(g_{0}),\qquad \bm{H}(t_i)  = \texttt{SGNN}(g_{i}),\qquad \bm{Q}(t_i)  = F(\bm{Q}(t_{i-1}),\bm{H}(t_i)),
\end{equation*}
where ${\bm{H}(t_i)\in \mathbb{R}^{N\times r},\,\bm{Q}(t_i)\in \mathbb{R}^{N\times s}}$, ${F:\mathbb{R}^{N\times s} \times \mathbb{R}^{N\times r}\rightarrow \mathbb{R}^{N \times s}}$ for $N$ nodes, an $r$-dimensional hidden space and an $s$-dimensional state space. Applying $F$ corresponds to component-wise applying $f$ for each node \cite{skarding2021foundations}.

In total, a function $\texttt{READOUT}_{\text{dyn}}$ will take as input $\bm{Q}(t_T)$ and return a suitable output for the considered task, so that altogether the CGNN is described as
\begin{equation*}
    \varphi(t_T, G, \set{V}) = \texttt{READOUT}_{\text{dyn}}(\bm{Q}(t_T)).
\end{equation*}

\end{definition}

%% file: chapter/preliminaries.tex
\section{Preliminaries}\label{section_preliminaries}

To establish the approximation theorem for continuous-time dynamic graphs, we first extend the notions of unfolding trees and the Weisfeiler-Lehman test to the continuous setting. These extensions provide the base for a constructive proof that the CGNNs defined above are universal approximators on continuous-time dynamic graphs, in probability and up to any precision.

\textbf{Unfolding trees} are derived from the node's neighborhood in the current graph, updated with all prior events. In the continuous-time dynamic setting, they form a sequence computed at event timestamps, where trees are empty if a node does not exist at a given time.

\begin{definition}[Continuous-Time Dynamic Unfolding Tree] \label{def_unfolding_tree_dynamic_cont}
Let $G = (g_0,\set{O})$ with start graph $g_0 = (\set{V}_0, \set{E}_0, \alpha_0, \omega_0)$ and events $\set{O}=\{o_i\}_{t_i\in\set{T}}$ be a CDG with timestamps $\set{T}\subset\mathbb{R}_0$. 
The \textbf{continuous-time dynamic unfolding tree} $\bm{T}_{v}^{(d)}(t_i)$ at time $t_i\in \set{T}$ of node $v\in\set{V}$ up to depth $d \in \mathbb{N}_0$ is defined as
\begin{align*}
    \bm{T}_{v}^{(d)} (t_i)= 
    \begin{cases}
         \texttt{Tree}(\alpha_{t_i}(v)), \quad \text{if } d = 0
        \\
        \texttt{Tree}\bigl(\alpha_{t_i}(v), \Omega_{t_i}( \set{N}_{t_i}(v)) , \bm{T}^{(d-1)}_{\set{N}_{t_i}(v)}(t_i)\bigr) \quad \text{if } d > 0,
    \end{cases}
    \end{align*}
\noindent
where ${\texttt{Tree}(\alpha_{t_i}(v))}$ consists of node $v$ with attribute $\alpha_{t_i}(v)$ at time $t_i$. Furthermore, $\texttt{Tree}\bigl(\alpha_{t_i}(v), \Omega_{t_i}( \set{N}_{t_i}(v)) , \bm{T}^{(d-1)}_{\set{N}_{t_i}(v)}(t_i)\bigr)$ is the tree rooted in node $v$ with attribute $\alpha_{t_i}(v)$ at time $t_i$. 

Additionally, $\bm{T}^{(d-1)}_{\set{N}_{t_i}(v)}(t_i) = \lmm{\bm{T}_{u_1}^{(d-1)}(t_i), \ldots, \bm{T}_{u_{|\set{N}_{t_i}(v)|}}^{(d-1)}(t_i)\rmm}$ are subtrees rooted in the neighborhood $\set{N}_{t_i}(v)$ at time $t_i$ connected to $v$ by edges with attributes $\Omega_{t_i}( \set{N}_{t_i}(v))$. If the node $v$ does not exist at time $t_i$, the corresponding tree is empty and $v$ does not occur in any neighborhood of other nodes. Then, the \textbf{dynamic unfolding tree of $v$ at time $t_i$},  ${\bm{T}_{v}(t_i) = \lim \limits_{d \rightarrow \infty} \bm{T}_{v}^{(d)}(t_i)}$ is defined as merge of all (possibly empty) trees $\bm{T}_{{v}}^{(d)}(t_i)$ for any $d$.
\end{definition}

The sequence of unfolding trees on continuous-time dynamic graphs enables a classification of nodes and graphs into equivalence classes as follows.

\begin{definition}[Continuous-Time Dynamic Unfolding Tree Equivalence]\label{def:unf_eq_cont}
Two nodes $u,v\in\set{V}\da \bigcup_{t_i\in\set{T}} \set{V}_{t_i}$ in a CDG $G$ are said to be \textbf{continuous-time dynamic unfolding tree (CUT) equivalent}, $u\sim_{CUT} v$, if and only if $\bm{T}_{u}(t_i)= \bm{T}_{v}(t_i)$ for every timestamp $t_i$.  

Analogously, two CDGs $G_1, \; G_2$ are \textbf{CUT equivalent}, $G_1 \sim_{CUT} G_2$, if and only if there exists a bijection between the nodes of the graphs that respects the partition induced by the CUT equivalence on the nodes.
\end{definition}


The equivalence between unfolding tree and 1-WL equivalence on static unattributed graphs, as shown in \cite{d2021unifying}, can be extended to CDGs using a continuous-time WL test aligned with our definition of continuous unfolding trees. This test generalizes the 1-dimensional dynamic Weisfeiler-Lehman (1-DWL) test variant from \cite{beddar2024weisfeiler} by applying color refinement to successively updated graphs. Evaluating neighborhood colorings at given timestamps allows for distinguishing non-isomorphic graphs based on their color sequences.

\begin{definition}[Continuous-Time Dynamic 1-WL Test]\label{def_WL_test_dynamic_cont}
Let $G = (g_0,\set{O})$ be a CDG with timestamps $\set{T}\subset\mathbb{R}_{\geq 0}$ and $g_{t_i}$ be the current graph at time $t_i$. 
Let further $\texttt{HASH}_0^{(t_i)}\da\set{A}\rightarrow\set{C}$ be an injective function that encodes node attributes of $g_{t_i}$ with a color from a color set $\set{C}$ and $\text{HASH}^{(t_i)}$ 
an injective function mapping an arbitrary number of colors at time $t_i$ to a new color. 

Then, the \textbf{Continuous-Time Dynamic 1-WL test (1-CWL)} generates a vector of color sets per timestamp $t_i\in \set{T}$ and iterations $j\in\mathbb{N}_0$ by:
\begin{itemize}[leftmargin=1.1cm]
    \item[\underline{$j=0:$}] The color of a node $v\in\set{V}$ is set to: 
    \begin{align*}
        c_{v}^{(0)}(t_i) =\begin{cases} \texttt{HASH}_0^{(t_i)}\left(\alpha_{v}(t_i)\right), & \text{if }v\in\set{V}_{t_i},\\
             c^\perp, & \text{otherwise.}
        \end{cases}
    \end{align*}
    
    \item[\underline{$j>0:$}] Then, the colors are updated by: 
    \begin{align*}
        c_{v}^{(j)}(t_i) = \texttt{HASH}^{(t_i)}\Bigl( \bigl[ c_{v}^{(j-1)}(t_i), \Omega_{t_i}\left(\set{N}_{t_i}(v)\right) , c_{\set{N}_{t_i}(v)}^{(j-1)}(t_i)\bigr]\Bigr) 
    \end{align*}
\end{itemize}
\end{definition}

\begin{definition}[Continuous-Time Dynamic WL Equivalence]\label{def_dynamic_cwl_equivalence}
Two nodes ${u,v\in \set{V}}$ in a CDG $G\in\set{G}$ are \textbf{CWL equivalent}, noted by $u \sim_{CWL} u$, if and only if their colors resulting from the 1-CWL test are pairwise equal per timestamp. 

Analogously, two CDGs $G_1, G_2\in\set{G}$ are \textbf{CWL equivalent}, $G_1 \sim_{CWL} G_2$, if and only if for all $v_1 \in \set{V}^{1}_{t_i}$ there exists a corresponding $v_2 \in \set{V}^{2}_{t_i}$ with ${c_{v_1}(t_i) = c_{v_2}(t_i)}$ for all $t_i\in \set{T}$.
\end{definition}

\begin{theorem}[CUT and CWL Equivalence for Nodes]\label{1-CWL=CUT}
   Let $G\in\set{G}$ be a continuous dynamic graph and $u,v \in \set{V}=\bigcup_{t_i\in\set{T}}\set{V}_{t_i}$. Then, it holds
\begin{align*} 
     u \sim_{CUT} v  \Longleftrightarrow u \sim_{CWL} v.
\end{align*} 
\end{theorem}

\begin{proof}
    Given that if two nodes $u,v\in\set{V}$ are attributed unfolding tree equivalent ${u \sim_{AUT} v}$, i.e., have the same unfolding tree in an attributed graph, the colorings of the corresponding attributed 1-WL test are equal and they are attributed 1-WL equivalent ${u \sim_{AWL} v}$ \cite[Lem.~4.1.5., Thm.~4.1.6]{beddar2024weisfeiler}. Therefore, two nodes in a CDG are CUT equivalent if and only if they are attributed unfolding tree equivalent ${u \sim_{AUT} v}$ at each timestamp ${ t_i \in \set{T}}$. Consequently, it holds that for all ${t_i \in \set{T}}$ the two nodes are attributed 1-WL equivalent $u \sim_{AWL} v$ and, thus, they are 1-dimensional continuous-time dynamic WL equivalent by Def.~\ref{def_dynamic_cwl_equivalence}. 
    In case of the non-existence of $u$ at a certain time step $t$, the theorem still holds.
\end{proof}
\medskip


In \cite{d2021unifying}, the authors demonstrate that for any pair of connected static graphs with $N$ nodes, it suffices to compare the unfolding trees up to depth $2N - 1$ to determine unfolding tree equivalence. Due to the correspondence of unfolding tree and WL equivalence, this result further provides a bound on the number of iterations required by the 1-WL test. 
However, since dynamic graphs are often disconnected at certain timestamps, it is important to examine whether this bound also applies in the disconnected case. 
In \cite{bamberger2022topological}, it has been shown that if two disconnected graphs are WL-indistinguishable, their connected components have the same stable WL colorings, and the components can be matched by a color-preserving bijection. 
This decomposition allows us to apply the $2N - 1$ bound locally to each connected component. As a result, the original bound can be tightened for disconnected graphs by applying it to their connected parts, each of which contains at most $N - 1$ nodes as follows. 

For the following lemma, the disjoint union of graphs is used, which is defined as the disjoint union of the nodes and edges, respectively:
\begin{align*}
    g_1 = g_2 \oplus g_3,\ \set{V}_1=\set{V}_2\cupdot\set{V}_3,\ \set{E}_1=\set{E}_2\cupdot\set{E}_3
\end{align*}

\begin{lemma}[Decomposition of Indistinguishable Graphs from \cite{bamberger2022topological}]\label{lem:decomposition_indist_graphs}
    Let $g$ and $\hat{g}$ be two graphs of bounded degree. Then the following two statements are equivalent:

    \begin{enumerate}
        \item $g$ and $\hat{g}$ are indistinguishable by the WL test\footnote{The lemma also holds for the attributed Weisfeiler-Lehman test introduced in \cite{beddar2024weisfeiler}, as its validity relies solely on the coloring scheme, not on the specific algorithm, mirroring the reasoning in \cite{bamberger2022topological}.}, i.e., there is a bijection between the connected components of $g$ and $\hat{g}$, such that the connected components $g_i$ and $\hat{g}_i$ have the same stable coloring $C_i\in\set{I}$. 
        \item $g$ and $\hat{g}$ admit decompositions into disjoint unions of connected components $g_{ij}$ and $\hat{g}_{ij}$:
        \begin{align*}
            g= \bigoplus\limits_{i\in\set{I}}\bigoplus\limits_{j\in\set{J}_i} g_{ij} \text{ and }
            \hat{g}= \bigoplus\limits_{i\in\set{I}}\bigoplus\limits_{\hat{j}\in\hat{\set{J}}_{i}} \hat{g}_{i\hat{j}}
        \end{align*}
        with the index sets of all connected components $\set{J}_i$ that have the same coloring $C_i$ for all $i\in\set{I}$, such that the following holds:
        \begin{itemize}
            \item for every $i\in\set{I}:\ |\bigoplus\limits_{j\in\set{J}_i} g_{ij}| = |\bigoplus\limits_{\hat{j}\in\hat{\set{J}}_{i}} \hat{g}_{i\hat{j}}|$
            \item for every $i\in\set{I}$, there is a graph $h_i$ that covers both $g_{ij}$ and $\hat{g}_{i\hat{j}}$ for all $j\in\set{J}$ and $\hat{j}\in\hat{\set{J}}$.
        \end{itemize}
        Moreover, if $|\bigoplus\limits_{j\in\set{J}_i} g_{ij}|<\infty$, the cover $h_i$ can be chosen to be finite.
    \end{enumerate}

%
\end{lemma}

Using the decomposition of indistinguishable graphs, the unfolding tree depth bounds can be directly derived for disconnected graphs. 
Intuitively, the threshold holds for the trees of every connected component of a disconnected graph and, thus, it holds for all trees in a disconnected graph and can, thus, be tightened.


\begin{theorem}[Unfolding Tree Depth]\label{thm:unfolding_tree_depth_disconnected}
    \begin{itemize}
        \item Let $x,y$ be nodes of graphs $g_1$ and $g_2$, respectively. Then, for their infinite unfolding trees $\bm{T}_x,\bm{T}_y$ it holds:
        \begin{align*}
            \bm{T}_x=\bm{T}_y \Longleftrightarrow \bm{T}_x^{(2N-1)}=\bm{T}_y^{(2N-1)}.
        \end{align*}
        \item Further, for any $N$, there exist two graphs $g_1,g_2$ with nodes $x,y$, respectively, such that the infinite unfolding trees $\bm{T}_x,\bm{T}_y$ are different, but equal up to depth $2N-16\sqrt{N}$. Specifically, $\bm{T}_x\neq\bm{T}_y$ and $\bm{T}_x^{(i)}=\bm{T}_y^{(i)}$ for $i\leq 2N-16\sqrt{N}$.
    \end{itemize}
\end{theorem}

\begin{proof} We consider connected and disconnected graphs separately:
Let $g_1$ and $g_2$ be connected graphs with the maximum number of nodes $N$. Then, the theorem is proven in \cite[Thm.~4.1.3]{d2021unifying}.

For disconnected graphs $g_1$ and $g_2$ with maximal number of nodes $N$, it holds:
    \begin{itemize}
        \item[\RM{1})] Let $x,y$ be nodes of $g_1$ and $g_2$, respectively. Then it is to be shown that for their infinite unfolding trees $\bm{T}_x,\bm{T}_y$ it holds:
        \begin{align*}
            \bm{T}_x=\bm{T}_y \Longleftrightarrow \bm{T}_x^{(2N-1)}=\bm{T}_y^{(2N-1)}.
        \end{align*}

        \item[\RM{2})] Further, it is to be shown that for any $N$, there exist two graphs $g_1,g_2$ with nodes $x,y$, respectively, such that the infinite unfolding trees $\bm{T}_x,\bm{T}_y$ are different, but equal up to depth $2N-16\sqrt{N}$. Specifically, $\bm{T}_x\neq\bm{T}_y$ and $\bm{T}_x^{(i)}=\bm{T}_y^{(i)}$ for $i\leq 2N-16\sqrt{N}$.
    \end{itemize}
    The idea is to prove the statements for each disconnected component of the graph. 
    Due to Lem.~\ref{lem:decomposition_indist_graphs}
    , two disconnected graphs that are indistinguishable by WL, can be decomposed into a finite set of disjoint connected components with a bijection between the graphs matching pairs of connected components that have the same coloring.
    
    Then, applying Lem.~\ref{lem:decomposition_indist_graphs} to two connected components with the same coloring, i.e., that are WL indistinguishable, we can conclude that they have isomorphic universal covers. From \cite{krebs2015universal} we know that if the universal covers are isomorphic, the universal covers up to depth $2N-1$ are isomorphic. Then, using the equivalence of isomorphism between universal covers and unfolding trees with the specific depth shown in \cite{d2021unifying}, the trees up to depth $2N-1$ are equal as stated in \RM{1}). 
    
    Further, in \cite{krebs2015universal} it is shown that for connected graphs $g_1,g_2$ of arbitrary size $N$ there is a pair of nodes $x,y$ of $g_1,g_2$, respectively, where the corresponding covers starting in $x$ and $y$ are isomorphic until depth $2N-16\sqrt{N}$ and non-isomorphic thereafter. Consequently, using again the correspondence to unfolding trees from \cite{d2021unifying}, the unfolding trees of the nodes are equal up to depth $2N-16\sqrt{N}$ and unequal thereafter, as stated in \RM{2}).
    
     As the graphs are disconnected, they consist of at least two non-empty connected components, each containing at most $N-1$ nodes. Consequently, the previously stated bounds can be formulated even tighter as follows:
    \begin{itemize}
        \item $\bm{T}_x=\bm{T}_y \Longleftrightarrow \bm{T}_x^{(2(N-1)-1)} = \bm{T}_x^{(2N-3)}=\bm{T}_y^{(2N-3)}=\bm{T}_y^{(2(N-1)-1)}$
        \item $\bm{T}_x\neq\bm{T}_y$ and $\bm{T}_x^{(i)}=\bm{T}_y^{(i)}$ for $i\leq 2(N-1)-16\sqrt{N-1}=2N-16\sqrt{N-1}-2$.
    \end{itemize}
\end{proof}

We now leverage the previously established results to constructively prove the approximation capability of CGNNs. This approach provides valuable insights into the required layer depth and the appropriate design of the temporal recurrent function of a CGNN.

%% file: chapter/apx_continuous_time.tex
\section{Universal Approximation Ability of CGNNs}\label{section_universal_apx_CGNN}
The dynamic system in \cite{beddar2024weisfeiler} processes a finite set of discrete time steps to reflect the computation of snapshots and their unfolding trees. In the continuous setting, the dynamic system is generalized to a finite set of arbitrary timestamps, which mirrors the processing of unfolding trees in continuous-time dynamic graphs.

 \begin{definition}[Continuous-Time Dynamic System]\label{def:cdyn_sys}
Let $\set{G}$ be the domain of bounded continuous-time dynamic graphs with total set of nodes $\set{V}$, $|\set{V}|=N$ and finite set of timestamps $\set{T}=\{t_0,\ldots,t_T\}$, and let ${\set{D}\da \set{T} \times \set{G} \times  \set{V}}$. 
A \textbf{continuous-time dynamic system}  ${\texttt{cdyn}\abb \set{D}\rightarrow \mathbb{R}^m}$ is then defined as
\begin{equation}\label{cdynsis}
    \texttt{cdyn}(t_i,G,v): = r(s_{t_i}(v))
\end{equation}
for $G = (g_0,\set{O}) \in \set{G},$  and $v \in \set{V}$. Here, $r:\mathbb{R}^{r} \rightarrow \mathbb{R}^m$ is an output function, and the state function $s_{t_i}(v)$ is determined by 
\begin{equation*}
    s_{t_i}(v)\; = \;
    \begin{cases}
      h(t_i,G, v)  &  \text{if}\; i=0 \\
      f(s_{t_{i-1}}(v), h(t_{i},G,v))  & \text{if} \; i>0,\\
    \end{cases}
\end{equation*}
for $v\in\set{V}$. Further, $h:\set{T} \times \set{G} \times  \set{V} \rightarrow \mathbb{R}^r$ is a function that processes the updated current graph $g_{t_i}$ at time $t_i$ and provides an $r$-dimensional internal state representation for each node $v$. Moreover, $f:\mathbb{R}^{r}\times \mathbb{R}^r \rightarrow \mathbb{R}^{r}$ is a recursive state update function that is (piece-wise) continuously differentiable (between consecutive events).
\end{definition}

\begin{definition}[Functions Preserving CUT Equivalence]\label{def:csys_unfold}
A continuous-time dynamic system $\texttt{cdyn}(\cdot,\cdot,\cdot)$ \textbf{preserves the continuous-time dynamic unfolding tree (CUT) equivalence} on $\set{G}$ if and only if for any CDGs $G_1,G_2\in \set{G}$, and two nodes $u, v\in\set{V}$ it holds 
\begin{align*}
    v \sim_{CUT} u  \Longrightarrow \texttt{cdyn}(t_i,G_1,v)=\texttt{cdyn}(t_i,G_2,u) \quad\forall t_i\in\set{T}.
\end{align*}
\end{definition}

The class of continuous-time dynamic systems that preserve the CUT equivalence on $\set{D}$ will be denoted with $\set{F}(\set{D})$. A characterization of $\set{F}(\set{D})$ is given by the following result (following the work in \cite{scarselli2008computational}).

\begin{proposition}[Functions of Continuous-Time Dynamic Unfolding Trees] \label{f_unfold_cdyn}
A continuous-time dynamic system \texttt{cdyn} belongs to $\set{F}(\set{D})$ if and only if there exists a function $\kappa$ defined on continuous-time dynamic unfolding trees such that for all ${(t_i,G,v)\in \set{D}}$ it holds 
\begin{equation*}
\text{\texttt{cdyn}}(t_i,G,v)= \kappa\Bigl(\bigl(\bm{T}_{v}(t_j)\bigr)_{t_j\leq t_i}\Bigr).
\end{equation*}
\end{proposition}

\begin{proof}
    Let $G\in\set{G}$.\\
\noindent\underline{$\Rightarrow:$} If \texttt{cdyn} preserves the CUT equivalence, then we can choose $\kappa$ as
    \begin{equation*}
        \kappa\Bigl(\bigl(\bm{T}_{v}(t_j)\bigr)_{t_j\leq t_i}\Bigr) = \texttt{cdyn}(t_i,G,v).
    \end{equation*}
This is equivalent to $u \sim_{CUT} v$ for $u,v\in G$ implies that ${\texttt{cdyn}(t_i,G,u) = \texttt{cdyn}(t_i,G,v)}$ and $\bigl(\bm{T}_{u}(t_j)\bigr)_{t_j\leq t_i}=\bigl(\bm{T}_{v}(t_j)\bigr)_{t_j\leq t_i}$. This directly defines $\kappa$ uniquely by
\begin{equation*}
    \kappa\Bigl(\bigl(\bm{T}_{u}(t_j)\bigr)_{t_j\leq t_i}\Bigr) = \kappa\Bigl(\bigl(\bm{T}_{v}(t_j)\bigr)_{t_j\leq t_i}\Bigr).
\end{equation*}    
\noindent\underline{$\Leftarrow:$} On the other hand, if there exists $\kappa$ such that ${\texttt{cdyn}(t_i,G,v) = \kappa\Bigl(\bigl(\bm{T}_{v}(t_j)\bigr)_{t_j\leq t_i}\Bigr)}$ for all $(t_i,G,v)\in \set{D}$, then for any pair of nodes $u,v \in G$ with $u \sim_{CUT} v$ it holds
    {\begin{equation*}
        \texttt{cdyn}(t_i,G,u) = \kappa\Bigl(\bigl(\bm{T}_{u}(t_j)\bigr)_{t_j\leq t_i}\Bigr) = \kappa\Bigl(\bigl(\bm{T}_{v}(t_j)\bigr)_{t_j\leq t_i}\Bigr) = \texttt{cdyn}(t_i,G,v).
    \end{equation*}}
\end{proof}


Note that in Def.~\ref{CGNN}, the temporal modeling function $f$ is not required to be continuously differentiable. In practice, $f$ may be piece-wise differentiable or even non-differentiable. Here, we focus on CGNNs where $f$ is piece-wise differentiable between events to analyze their universal approximation capabilities.

\begin{theorem}[Universal Approximation Theorem by CGNN]\label{cdyn_thm_approximation}
Let $G \in\set{G}$ be a CDG with node set bounded by $N$. 
Further, let $\texttt{cdyn}(t_i,G,v) \in \set{F}(\set{D})$ be any measurable continuous-time dynamic system preserving the CUT equivalence,  $\| \cdot \|$ be a norm on $\mathbb{R}$, $P$ be any probability measure on $\set{D}$ and ${\epsilon, \lambda \in\mathbb{R},\ \epsilon,\lambda >0}$. 
Then, there exists a CGNN composed by an SGNN with $2N-1$ layers and hidden dimension $r=1$, a temporal function that is (piece-wise) continuously differentiable (between consecutive events) and realized by an RNN with state dimension $s=1$, such that the function $\varphi$ realized by this model satisfies 
\begin{equation}\label{eq:cdyn_apx}
P( \| \texttt{cdyn}(t_i,G,v)- \varphi(t_i, G,v) \| \leq \varepsilon) \geq 1- \lambda \qquad \forall t_i \in \set{T}.
\end{equation} 
\end{theorem}

\begin{proof}
By assumption, $f$ used in \texttt{cdyn} is piece-wise continuously differentiable between consecutive events $o_{t_{i-1}},o_{t_{i}}\!\in\!\set{O}$ for all $i\!\in\![T]$, i.e., for all  $v\in\set{V}$ it holds
\begin{equation*}
    f(\bm{q}_v(t_{i-1}), \bm{h}_v(t_{i})) = \mathbb{1}_{t_{0}}\bm{h}_v(t_{i})+\sum\limits_{j\in [T]}\mathbb{1}_{(t_{j-1},t_{j}]}f_j(\bm{q}_v(t_{i-1}), \bm{h}_v(t_{i})) =: \bm{q}_v(t_{i})
\end{equation*}
with $f_j\abb \mathbb{R}^{s\times r}\rightarrow\mathbb{R}^s$ being continuously differentiable in the inter-event time intervals $(t_{j-1},t_{j}]$, hidden states $\bm{h}_v(t_{i}) = \texttt{SGNN}(g_i)$ and indicator function $\mathbb{1}$. For $i=0$, we have ${\bm{q}_v(t_{i}) \da \bm{h}_v(t_{i}) = \texttt{SGNN}(g_0)}$.
Then, the continuous dynamic system with output function $r$ can be rewritten as 
\begin{align*}
    \texttt{cdyn}(t_{i}, G, v) &= r(s_{t_{i}}(v)),\text{ with } \\
    s_{t_{i}}(v) &= \mathbb{1}_{t_{0}}\bm{h}_v(t_{i})+\hspace*{-.25cm}\sum\limits_{j\in [T]}\hspace*{-.1cm}\mathbb{1}_{(t_{j-1},t_{j}]}f_j(\bm{q}_v(t_{i-1}), \bm{h}_v(t_{i})).
\end{align*}
Therefore, we can find continuous-time dynamic systems \texttt{cdyn}$_j\in\set{F}(\set{D})$ such that
\begin{align}\label{sum_cdyn}
    \texttt{cdyn}(t_{i}, G, v) &= \begin{cases}
        \mathbb{1}_{t_0} r(h_{t_0}(v)), & \text{if } i = 0,\\
        \mathbb{1}_{(t_{i-1}, t_i]} r(s_{t_i}(v)), & \text{if } i > 0,
    \end{cases}\nonumber\\
    & = \mathbb{1}_{t_0} \texttt{cdyn}_0(t_0,G,v) + \sum\limits_{j\in[T]} \mathbb{1}_{(t_{j-1}, t_j]} \texttt{cdyn}_j (t_{i},G,v)
\end{align}
Let $\varphi$ be the realization of the \texttt{CGNN} as in the statement. 
Then, we can rewrite the stacked temporal function $F$ for $i>0$ as 
\begin{align*}
    F(\bm{Q}(t_{i-1}),\bm{H}(t_i))=&\sum\limits_{j\in[T]} F_j(\bm{Q}(t_{i-1}),\bm{H}(t_i))  =: \bm{Q}(t_i), \text{ with }\\
     &F_j(\bm{Q}(t_{i-1}),\bm{H}(t_i))\da \mathbb{1}_{(t_{j-1}, t_j]} F(\bm{Q}(t_{i-1}),\bm{H}(t_i))
\end{align*}
defined on intervals $(t_{j-1}, t_j]$, and hidden representations $\bm{H}(t_i) \da \texttt{SGNN}(g_i)$. For $i=0$, we have $\bm{Q}(t_0)\da SGNN(g_0).$ 
For all $j\in[T]$, let \texttt{CGNN}$_j$ be defined by the temporal functions $F_j$ and the \texttt{SGNN} as above. Since $\varphi$ is the realization of the previously defined \texttt{CGNN}, there exist $\varphi_j$ defined on $(t_{j-1}, t_j]$ realized by the \texttt{CGNN}$_j$, so that 
\begin{align}\label{sum_varphis}
    \varphi(t_i, G,v) = \mathbb{1}_{t_0}\varphi_0(t_i, G,v) + \sum_{j\in[T]}\mathbb{1}_{(t_{j-1}, t_j]} \varphi_j(t_i, G,v).
\end{align}

In the following, we show that \textbf{the statement holds for CGNNs with continuously-differentiable temporal functions \hypertarget{dagger}{($\bm{\dagger}$)}}, and in particular, that for all $\bar{t}\da (t_{i-1}, t_i]$, the \texttt{CGNN}$_i$ can approximate any continuous-time dynamic system on $\bar{t}$ in probability, up to any degree of precision.

The proof of the approximation capability of \texttt{CGNN}$_i$ largely follows the argument for discrete-dynamic GNNs in \cite[Thm.~5.2.4]{beddar2024weisfeiler}, with the main difference being the explicit use of the timestamp in the SGNN at time $t_i$. \\
The proof proceeds analogously to \cite[Apx.~A.5]{beddar2024weisfeiler}, except for the graph partitioning step in \cite[Lem.~A.5.1]{beddar2024weisfeiler}, which directly involves time steps. To adapt this to the continuous-time dynamic setting, we append the timestamps to the node attributes, allowing the domain $\set{G}$ to be partitioned into graphs with similar structure, attributes, and timestamps. This modification ensures that conditions (2)–(6) in the lemma still hold, enabling the remainder of the proof to continue as valid by replacing discrete steps with timestamps. Thus, the statement ($\bm{\dagger}$) holds for all $i \in [T]$.

This serves as an auxiliary statement that we can now apply to any time interval before $i\in[T]$ and any node $v\in\set{V}$:
\begin{align*}
    &\| \texttt{cdyn}(t_i,G,v)- \varphi(t_i, G,v) \| \\
    &\overset{\eqref{sum_cdyn}}{=} \|\Bigl(\mathbb{1}_{t_0} \texttt{cdyn}_0(t_i,G,v) + \sum\limits_{j\in[T]} \mathbb{1}_{(t_{j-1}, t_j]} \texttt{cdyn}_j (t_{i},G,v)\Bigl) - \varphi(t_i,G,v)\|\\
    &\overset{\eqref{sum_varphis}}{=} \|\Bigl(\mathbb{1}_{t_0} \texttt{cdyn}_0(t_i,G,v) + \sum\limits_{j\in[T]} \mathbb{1}_{(t_{j-1}, t_j]} \texttt{cdyn}_j (t_{i},G,v)\Bigr) \\
    & \quad\quad - \Bigl(\mathbb{1}_{t_0}\varphi_0(t_i, G,v) + \sum_{j\in[T]}\mathbb{1}_{(t_{j-1}, t_j]} \varphi_j(t_i, G,v\Bigr)\|\\
    & =  \|\mathbb{1}_{t_0}\left(\texttt{cdyn}_0(t_i,G,v) - \varphi_0(t_i, G,v)\right) 
     + \hspace*{-.3cm}\sum_{j\in[T]}\hspace*{-.3cm}\mathbb{1}_{(t_{j-1}, t_j]} \left(\texttt{cdyn}_j (t_{i},G,v) - \varphi_j(t_i, G,v)\right) \|\\
    &\overset{\triangle-ineq.}{\leq}\hspace*{-.3cm} \mathbb{1}_{t_0} \|\texttt{cdyn}_0(t_i,G,v) - \varphi_0(t_i, G,v)\| 
    + \hspace*{-.3cm}\sum_{j\in[T]}\hspace*{-.3cm}\mathbb{1}_{(t_{j-1}, t_j]} \|\texttt{cdyn}_j (t_{i},G,v) - \varphi_j(t_i, G,v) \|\\
    &\overset{\hyperlink{dagger}{\textbf{($\bm{\dagger}$)}}}{\leq} \frac{\epsilon}{T} + (T-1)\cdot\frac{\epsilon}{T} = \epsilon
\end{align*}

\noindent So in probability, it follows
\begin{align*}
     P(\| \texttt{cdyn}(t_i,G,v)- \varphi(t_i, G,v) \| \leq \epsilon)
    &\geq \sum\limits_{j\in[T]} P(\| \texttt{cdyn}_j(t_i,G,v)- \varphi_j(t_i, G,v) \| \leq \frac{\epsilon}{T})\\
    &\overset{\hyperlink{dagger}{\textbf{($\bm{\dagger}$)}}}{\geq}  \sum\limits_{j\in[T]} (1-\frac{\lambda}{T}) = T - \lambda \geq 1-\lambda.
\end{align*}
\end{proof}

\begin{remark}
    In continuous-time dynamic graphs, structural changes over time can lead to disconnected graphs. For example, when nodes are added without immediate connections or when nodes or edges are removed. It is therefore essential to account for disconnected graphs in the theoretical analysis of the approximation capabilities. While prior results on unfolding tree depth assume connected graphs \cite{d2021unifying}, we have shown in Thm.~\ref{thm:unfolding_tree_depth_disconnected} that the same bounds apply to disconnected graphs. This extension ensures that our proof of the universal approximation capability of CGNNs in Thm.~\ref{cdyn_thm_approximation} holds for all bounded continuous-time dynamic graphs, regardless of connectivity.
\end{remark}

Now we examine the approximation capabilities of CGNNs consisting of universal components, as defined below.

\begin{definition}[Universal Components] \label{def:universal_static}
Let $f^{(i)}$ be a transition function indicating the concatenation of the $\texttt{COMB}^{(i)}$ and $\texttt{AGGR}^{(i)}$ functions, i.e.,
\begin{equation*} 
     f^{(i)}\left[\bm{h}_v^{(i)},\{\bm{h}^{(i-1)}_u\}_{u \in \set{N}_v},\; \Omega'(\set{N}_v) \right]  
    = \  \texttt{COMB}^{(i)}\left[ \bm{h}_v^{(i-1)}, \texttt{AGGR}^{(i)}\left(\{ \bm{h}_u^{(i-1)}\}_{ u \in \set{N}_v}, \Omega'(\set{N}_v)  \right)\right],
\end{equation*}
denoted as $\bm{q}^{(i)}$ for every iteration $i\in[L]$.
Then, a class $\set{Q}_S$  of SGNN models has \textbf{universal components} if, for any $\epsilon>0$
and continuous target functions $\overline{\texttt{COMB}}^{(i)}$, $\overline{\texttt{AGGR}}^{(i)}$,  $\overline{\texttt{READOUT}}$, 
there exists an \texttt{SGNN} $\in\set{Q}_S$, with functions $\texttt{COMB}_\theta^{(i)}$, $\texttt{AGGR}_\theta^{(i)}$, $\texttt{READOUT}_\theta$ and parameters $\theta$, so that for any embedding $\bm{h}\in \mathbb{R}^r$ it holds
\begin{align*}
\left\|{\bar f}^{(i)}(\bm{h}_v,\{ \bm{h}_{u}\}_{u\in\set{N}_v}, \Omega'(\set{N}_v)) - f_\theta^{(i)}(\bm{h}_v,\{ \bm{h}_{u}\}_{u\in\set{N}_v}, \Omega'(\set{N}_v)) 
\right\|_\infty \leq \epsilon \\
\left\| \overline{\texttt{READOUT}}( \bar{\bm{q}}^{(L)})-\texttt{READOUT}_\theta( \bm{q}_\theta^{(L)})\right\|_\infty \leq  \epsilon\, ,
\end{align*}
The transition functions ${\bar f}^{(i)}$, $f_\theta^{(i)}$ and outputs $\bar{\bm{q}}^{(L)}, \bm{q}_\theta^{(L)}\in \mathbb{R}^r$ after $L$ iterations correspond to the target function and the \texttt{SGNN}, respectively.
\end{definition}

\begin{definition}[CGNNs with Universal Components] \label{def:universal_cdyn}
A class $\set{Q}_{\set{D}}$ of CGNN models has \textbf{universal components} if and only if for any continuously differentiable target functions 
$\overline{f}$, $\overline{\texttt{READOUT}}_{\text{cdyn}}$ there is a \texttt{CGNN}$_\theta\in\set{Q}_{\set{D}}$ with parameters $\theta$, utilizing an \texttt{SGNN} that has universal components, a continuously differentiable recurrent function $f_\theta$ and an output function $\texttt{READOUT}_\theta$ such that for any $\epsilon_1,\epsilon_2 > 0$ and any hidden representations $\bm{h}\in\mathbb{R}^r$, $\bm{q},\bm{q}^L\in\mathbb{R}^s$ and number of layers $L$ it holds:
\begin{align*}
\left\| \overline{f}( \bm{q}, \bm{h})-f_\theta( \bm{q},\bm{h})\right\|_\infty &\leq \epsilon_1, \; \\
\left\| \overline{\texttt{READOUT}}_{\text{cdyn}}( \bm{q}^L)-\texttt{READOUT}_{\text{cdyn},\theta}( \bm{q}^L)\right\|_\infty &\leq \epsilon_2.
\end{align*}
\end{definition}

Then, the following theorem indeed establishes that CGNNs with universal components approximate CGNNs that fulfill Thm.~\ref{cdyn_thm_approximation} appropriately.

\begin{theorem}[Approximation by Neural Networks]\label{cdyn_mainNN}
	Assume that the hypotheses of Thm.~\ref{cdyn_thm_approximation} are fulfilled and 
	$\set{Q}_{\set{D}}$ is the class of CGNNs with universal components.
	Then, there exists a parameter set $\theta$, and the functions 
	$\overline{f}$, $\overline{\texttt{READOUT}}_{\text{cdyn}}$,  implemented by Neural Networks in $\set{Q}_{\set{D}}$, such that Thm.~\ref{cdyn_thm_approximation} holds. 
\end{theorem}

\begin{proof}[Sketch of the Proof.]
    Based on Thm.~\ref{cdyn_thm_approximation}, we know that there exists a $\overline{\texttt{CGNN}}$ capable of universally approximating functions in $\set{F}(\set{D})$. Adapting \cite[Thm.~A.5.2]{beddar2024weisfeiler} to timestamps, this domain can be assumed to be finite and bounded, and the (piece-wise) differentiable functions in $\overline{\texttt{CGNN}}$ are also bounded (between consecutive events) and, thus, have bounded Jacobians. Consequently, we can construct a \texttt{CGNN}$_\theta \in \set{Q}_{\set{D}}$ by assembling universal \texttt{CGNN} components for each inter-event interval. These components approximate the respective parts of $\overline{\texttt{CGNN}}$ closely enough so that the entire \texttt{CGNN}$_\theta$ remains a universal approximator on the domain $\set{G}$.
\end{proof}

\newpage

%% file: chapter/conclusion.tex
\section{Discussion and Conclusion}\label{section_conclusion}
This paper advanced the theoretical analysis of GNN expressivity on discrete-time to continuous-time dynamic graphs (CTDGs) with general connectivity. We introduced a continuous-time extension of the 1-WL test along with corresponding unfolding trees (Sec.\ref{section_preliminaries}) and showed that both characterizations yield equivalent graph partitions in the continuous-time setting. Based on this correspondence, we provided a constructive proof that a broad class of Continuous-Time Dynamic GNNs (CGNNs) matches the distinguishing power of the continuous-time 1-WL and can universally approximate any measurable function that respects continuous-time unfolding tree equivalence (Sec.\ref{section_universal_apx_CGNN}). 

Our constructive analysis reveals that the static GNN component of a universal approximator CGNN requires only a single hidden dimension and at least $2N-1$ layers to operate on graphs with up to $N$ nodes. Moreover, the recurrent update function can also be realized with a one-dimensional state. We also demonstrated that it is sufficient for this temporal function to be piece-wise continuously differentiable between successive events to preserve the CGNN’s universal approximation capability.

Future work may extend the expressivity analysis to alternative CGNN architectures beyond the specific combination of static SGNNs and temporal update functions considered here. Moreover, while our results focus on the expressive power, models with equivalent expressivity may differ significantly in terms of efficiency, scalability, and generalization, highlighting the need to better understand how architectural choices in \texttt{AGGR}, \texttt{AGGR}, and \texttt{READOUT} affect these practical aspects. The smooth 1–WL test and its unfolding-tree perspective will provide a solid foundation for empirical studies on large-scale benchmarks and inspire further refinements through novel temporal encodings or attention mechanisms.